%% file: 0859.tex
\pdfoutput=1
\documentclass[10pt,twocolumn,letterpaper]{article}

\usepackage{wacv}
\usepackage{times}
\usepackage{epsfig}
\usepackage{graphicx}
\usepackage{amsmath}
\usepackage{amssymb}

\def\wacvPaperID{0859} %

\wacvfinalcopy %
\ifwacvfinal
\def\assignedStartPage{1} %
\fi

\ifwacvfinal
\usepackage[breaklinks=true,bookmarks=false]{hyperref}
\else
\usepackage[pagebackref=true,breaklinks=true,colorlinks,bookmarks=false]{hyperref}
\fi

\usepackage{graphicx}
\usepackage{xcolor}
\usepackage{hyphenat}
\usepackage{enumitem,kantlipsum}
\usepackage{hyperref}       %
\usepackage{url}            %
\usepackage{booktabs}       %
\usepackage{amsfonts}       %
\usepackage{nicefrac}       %
\usepackage{microtype}      %
\usepackage{amssymb}
\usepackage{amsmath}
\usepackage{subfigure}
\usepackage{lipsum,multicol}
\usepackage{multirow}
\usepackage{array}
\usepackage{amsthm}
\ifwacvfinal
\else
\fi

\newcommand{\CW}{\mathcal{W}}
\newcommand{\BE}{\mathbb{E}}
\newcommand{\CL}{\mathcal{L}}
\newcommand{\CT}{\mathcal{T}}
\newcommand{\CG}{\mathcal{G}}
\newcommand{\bw}{\mathbf{w}}
\newcommand{\bx}{\mathbf{x}}
\newcommand{\bu}{\mathbf{u}}
\newcommand{\by}{\mathbf{y}}
\newcommand{\meta}{\text{meta}}
\newcommand{\neta}{\text{noisy-meta}}
\newcommand{\train}{\text{train}}

\newcommand{\ce}{\text{CE}}
\newcommand{\mae}{\text{MAE}}

\DeclareMathOperator*{\argmin}{\arg\!\min}

\newtheorem{theorem}{Theorem}
\newtheorem{lemma}[theorem]{Lemma}
\theoremstyle{remark}
\newtheorem*{remark}{Remark}
\begin{document}

\title{Do We Really Need Gold Samples for Sample Weighting under Label Noise?}

\author{Aritra Ghosh\\
University of Massachusetts Amherst\\
Amherst, MA, USA\\
{\tt\small arighosh@cs.umass.edu}
\and
Andrew Lan\\
University of Massachusetts Amherst\\
Amherst, MA, USA\\
{\tt\small andrewlan@cs.umass.edu}
}

\maketitle
\thispagestyle{empty}

\input{abs}
\input{intro}

\input{problem}

\input{method}

\input{exp}
\input{related}

\input{con}

{\small
\bibliographystyle{ieee_fullname}
\bibliography{egbib}
}
\input{supp}

\end{document}

%% file: abs.tex
\begin{abstract}
    Learning with labels noise has gained significant traction recently due to the sensitivity of deep neural networks under label noise under common loss functions. Losses that are theoretically robust to label noise, however, often makes training difficult. Consequently, several recently proposed methods, such as Meta-Weight-Net (MW-Net), use a small number of unbiased, clean samples to learn a weighting function that downweights samples that are likely to have corrupted labels under the meta-learning framework. 
    However, obtaining such a set of clean samples is not always feasible in practice.
    In this paper, we analytically show that one can easily train MW-Net without access to clean samples simply by using a loss function that is robust to label noise, such as mean absolute error, as the meta objective to train the weighting network. We experimentally show that our method beats all existing methods that do not use clean samples and performs on-par with methods that use gold samples on benchmark datasets across various noise types and noise rates. 
\end{abstract}

%% file: intro.tex
\section{Introduction}
Although deep neural networks (DNNs) have achieved impressive performance across many applications, they remain vulnerable to noisy labels. The  expressive power of DNNs enables them to learn arbitrary smooth non-linear functions; however, it also means that DNNs can easily overfit to noisy labels, which negatively affects their generalization ability. Learning under label noise is of utmost importance in the big-data regime since large amounts of labels are collected using crowd-sourcing, resulting in significant label noise \cite{reed,xiao2015learning,webly,veit2017learning,reed}; see \cite{survey,survey-new,github} for detailed surveys and recent results on learning under label noise. 
Methods for mitigating label noise can be broadly categorized into two groups based on whether they use a small number of clean samples in the learning process. 

In the first group, most algorithms do not assume access to any clean samples; instead, they use loss functions that are robust to label noise to handle noisy labels implicitly \cite{unhinged,ghosh2017,generalized-ce}. 
Unfortunately, the common cross entropy loss used to train DNNs is particularly sensitive to noisy labels. 
As an alternative, in \cite{ghosh2017}, the authors have shown that the non-convex mean absolute error (MAE) loss is robust to label noise under some mild assumptions. However, training DNNs under the MAE loss on large datasets is often difficult. 
In \cite{generalized-ce}, the authors proposed generalized cross-entropy loss, a generalization of MAE loss and cross-entropy loss, that provides an effective trade-off between theoretical guarantees and the ease of DNN training \cite{generalized-ce}. However, the performance of these theoretically robust losses on large-scale image datasets is often limited.

In the second group, most methods use a small number of clean samples to compute the noise transition matrix, predict which samples are noisy, and correct the loss functions by learning a sample re-weighting strategy \cite{forward,reed,mwnet}. 
One method is to use a loss correction approach \cite{natarajan} when the noise rates are known or can be accurately estimated from data. 
Another approach is to use sample re-weighting to downweight samples for which the classifier incurs large losses on and are thus likely to be noisy. 
Many early methods use heuristics to compute sample weights.  
However, a pre-determined weighting scheme is not very effective and cannot leverage real data. Therefore, many recent methods proposed to adaptively learn the sample weighting function from the additional clean data \cite{mwnet,l2rw,mentornet}. Meta-Weight-Net (MW-Net) \cite{mwnet}, a highly effective method, uses recent advances in \emph{meta-learning} to jointly learn the classifier network and the weighting network. MW-Net uses a small number of clean samples as meta samples to learn the sample weighting strategy used for the classifier network via the meta objective. Experimentally, MW-Net achieves impressive performance on real-world, large-scale image datasets under label noise \cite{mwnet}. 
However, the most apparent drawback of MW-Net and other methods in this group is that we may not have access to clean samples in real-world applications.

\textbf{Contributions}
We make a surprising observation that it is very easy to adaptively learn sample weighting functions, even when we do not have access to any clean samples; we can use \emph{noisy} meta samples to learn the weighting function if we simply change the meta loss function. We theoretically show that the meta-gradient direction for optimizing the weighting network remains the same when using noisy meta samples with meta losses that are robust to label noise such as the MAE \cite{ghosh2017}. Thus, we can optimize the classifier network using the cross-entropy loss and optimize the weighting network using the MAE loss, both with noisy samples. 
Although the MAE loss is difficult to optimize for deep classification networks \cite{generalized-ce}, they are easy to optimize for shallow networks that are used as the weighting network \cite{ghosh2017}. Our method closes the gap on learning the weighting function adaptively without the need to access clean meta samples; this setup closely resembles real-world problem settings where clean samples are often missing.

First, we theoretically show that it is \emph{possible} to learn the weighting function even without any access to clean meta samples. Then, we experimentally show that our proposed method (using noisy meta samples) performs on par with MW-Net (using clean meta samples) and beats existing methods on several benchmark image datasets. Thus, the simple observation that the meta-gradient direction remains the same for noisy meta samples under meta losses that are robust to label noise alleviates the need for clean samples when learning sample re-weighting strategy under label noise. 
\footnote{Our code is available at \url{https://github.com/arghosh/RobustMW-Net}.}

%% file: problem.tex
\section{Problem Setup and Label Noise}
We consider the classification problem setting with a training set $\{(\bx_i^{\text{train}},\by_i^{\text{train}})\}_{i=1}^N$, a meta (or validation) set $\{(\bx_j^{\text{meta}},\by_j^{\text{meta}})\}_{j=1}^M$, and a test set $\{(\bx_o^{\text{test}},\by_o^{\text{test}})\}_{o=1}^O$ where $\bx_i$ is the feature vector of the $i^{\text{th}}$ sample and $\by_i\in \{0,1\}^K$ is the class label vector with $K$ total classes. We denote the classifier  network prediction as $f(\bx, \bw)$ where $\bw$ is the classifier network weights. 
We denote the clean (unknown) label for an instance $\bx_i$ as $\by_i^{\text{true}}$. We consider label noise where feature vectors are clean but the labels vectors are corrupted \cite{natarajan}. We consider three types of label noise. The uniform (or symmetric) noise with rate $\eta\in [0,1]$, corrupts the true label with probability $\eta$ uniformly distributed among all classes. Thus, the true class label remains the same with probability $(1-\eta)+\frac{\eta}{K}$ and the true label is corrupted to all other classes $c\neq \by_i^{\text{true}}$ with probability $\frac{\eta}{K}$.
The flip noise model, with noise rate $\eta$, corrupts the true labels with probability $\eta$ to a random single class.  Thus, the correct class label remains the same with probability $(1-\eta)$ and the true label is corrupted to a single class $c\neq \by_i^{\text{true}}$ with probability ${\eta}$.
The flip2 noise model, with noise rate $\eta$, corrupts the true labels with probability $\eta$ to two random other classes.  Thus, the correct class label remains the same with probability $(1-\eta)$ and the true label is corrupted to two other classes $c_1,c_2\neq \by_i^{\text{true}}$ with probability $\frac{\eta}{2}$. Figure~\ref{fig:noise} shows example confusion matrices for these three noise models. 
	\begin{figure}[pt]%
			\includegraphics[width=0.49\textwidth]{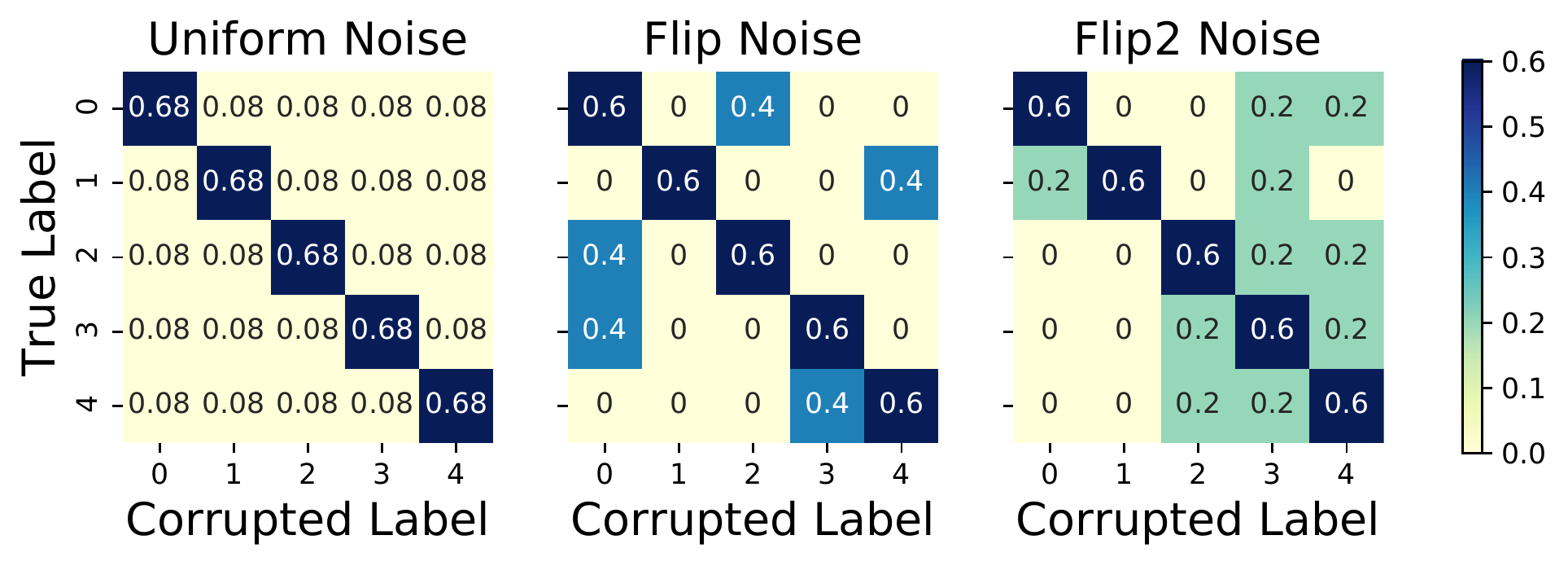} 
		\caption{Confusion matrices for the noise models (K=5,$\eta$=0.4). 
		} \label{fig:noise}
		\vspace{-0.1in}
	\end{figure}

We assume that the training dataset $\{(\bx_i^{\text{train}},\by_i^{\text{train}})\}_{i=1}^N$ is corrupted with an unknown noise model, but the test set $\{(\bx_o^{\text{test}},\by_o^{\text{test}})\}_{o=1}^O$ remains clean. 
Deep networks are often optimized with the cross-entropy (CE) loss; we obtain the optimal classifier network  parameter $\bw^{\ast}$ by minimizing CE loss on the training dataset as
\begin{align*}
\resizebox{0.43\textwidth}{!}{$
\bw^{\ast} =\argmin_{\bw}  \frac{1}{N}\sum_{i=1}^N \ell_{\ce}(\by_i^{\text{train}}, f(\bx_i^{\text{train}},\bw))
$}
\vspace{-0.2in}
\end{align*}
where $\ell_{\ce}$ is the CE loss. 

In sample weighting methods, a common approach is to introduce weighting function $\CW(\bx_i^{\text{train}},\by_i^{\text{train}},\bw; \Theta)$, which is (optionally) parameterized by $\Theta$ to decide the weight for $i^{\text{th}}$ training sample. 
In MW-Net, the function $\CW$ takes sample loss as the input \cite{mwnet}; thus, we can represent the sample weighting function as $\CW(\ell_{\ce}(\by_i^{\text{train}}, f(\bx_i^{\text{train}},\bw));\Theta)$ where $\ell_{\ce}(\by_i^{\text{train}}, f(\bx_i^{\text{train}}, \bw))$ is  the current sample loss and $\Theta$ is the (meta) weight network parameter. 
For simplicity, throughout the paper, we use $\ell(\by_i^{\text{train}},f(\bx_i^{\text{train}}),\bw))$ as $\ell^{i,{\text{train}}}(\bw)$ and $\ell(\by_j^{\text{meta}},f(\bx_j^{\text{meta}}),\bw))$ as $\ell^{j,{\text{meta}}}(\bw)$.
Thus, we can calculate classifier network weights $\bw^{\ast}$ under label noise with the sample weighting parameter $\Theta$ as \cite{mwnet}:
\begin{align*}
\resizebox{0.4782\textwidth}{!}{$
 \argmin_{\bw}  \CL^{\text{train}}(\bw;\Theta)
\triangleq \frac{1}{N}\sum_{i=1}^N \CW(\ell_{\ce}^{i,\text{train}}(\bw);\Theta) \ell_{\ce}^{i,{\text{train}}}(\bw).
$}
\vspace{-0.2in}
\end{align*}

%% file: method.tex
\section{Methodology}
We first discuss the basic setup of MW-Net \cite{mwnet}. We then show that it is possible to train a robust model without access to clean meta samples (thus, $\by_j^{\text{meta}}$ can be noisy as well). 
\paragraph{Meta-Weight-Net.}
MW-Net learns an auxiliary neural network with parameter $\Theta$ for parameterizing the weighting function \cite{mwnet}. Many previous heuristics employ tricks such as removing training samples with a high loss to minimize the effect of noisy samples. Instead of using some fixed heuristics, we can learn a data-driven adaptive non-linear weighting function. MW-Net is an effective way to learn the weighting function using ideas from meta-learning. The idea is to use a clean meta dataset (or validation dataset) $\{\bx_j^{\text{meta}},\by_j^{\text{meta}}\}_{j=1}^M$ ($M\ll N$) 
to learn the optimal weighting function $\Theta^{\ast}$ in a bilevel optimization problem \cite{bilevel}.
\begin{align}
&\min_{\Theta} \CL^{\text{meta}}(\bw^{\ast}(\Theta))\triangleq\frac{1}{M}\sum_{j=1}^M \ell^{j,\text{meta}}(\bw^{\ast}(\Theta)) \label{eq:theta-ast}\\
&\mbox{s.t. } \bw^{\ast}(\Theta)=\argmin_{\bw} \frac{1}{N}\sum_{i=1}^N \CW\big(\ell_{\ce}^{i,\train}(\bw);\Theta\big) \ell_{\ce}^{i,\train}(\bw) \nonumber
\end{align}
Thus, the inner objective is to optimize $\bw^{\ast}$ on $\CL^{\text{train}}$ and the outer objective is to optimize for $\Theta^{\ast}$ on $\CL^{\text{meta}}$ with nested loops. In the outer level optimization, the weighting network computes meta loss $\CL^{\meta}$ using the loss function $\ell^{\cdot,\meta}$; for brevity, we refer $\ell^{\cdot,\meta}$ as the \emph{loss function of the weighting network}  or the \emph{meta loss function} throughout the paper.

In \cite{mwnet}, the authors proposed a bilevel online optimization strategy to iteratively optimize both $\bw$ and $\Theta$ \cite{bilevel}. 
The gradient for $\Theta$, in Eq.~\ref{eq:theta-ast} at time step $t$, requires the function $\hat{\bw}(\Theta)$. The classifier weight function $\hat{\bw}^t(\Theta)$ (note, this is a function of $\Theta$, not the weight parameter) is computed as:
\[\hat{\bw}^t(\Theta) = \bw^t - \alpha\frac{1}{n}\sum_{i=1}^n \CW(\ell_{\ce}^{i,\train}(\bw^t);\Theta) \nabla_{\bw} \ell_{\ce}^{i,\train}(\bw)\Big|_{\bw^t} \]
where $n$ is the minibatch size for the training dataset, and $\alpha$ is the SGD optimizer parameter. 
Using the function $\hat{\bw}^t(\Theta)$, we can take a gradient step for $\Theta^t$ based on the loss on the meta dataset in Eq.~\ref{eq:theta-ast}. 
\begin{equation}
\Theta^{t+1} =\Theta^t-\beta\frac{1}{m}\sum_{j=1}^m\nabla_{\Theta}\ell^{j,\text{meta}}(\hat{\bw}^t(\Theta))\Big|_{\Theta^t}
\label{eq:meta-update}
\end{equation}
where $\beta$ is the weighting network SGD parameter and $m$ is the minibatch size for the meta dataset.  Using the updated $\Theta^{t+1}$, we can take a step towards optimal $\bw^{\ast}$ using simple stochastic gradient descent step
\[\bw^{t+1} \!\! = \!\bw^t \!- \alpha\frac{1}{n}\!\sum_{i=1}^n \CW(\ell_{\ce}^{i,\train}(\bw^t);\Theta^{t+1})\nabla_{\bw}\ell_{\ce}^{i,\train}(\bw)\Big|_{\bw^t}.\]
In essence, MW-Net takes gradient descent step for $\Theta^t$ towards $\Theta^{\ast}$ and takes gradient descent step for $\bw^t$ towards $\bw^{\ast}$ using the updated $\Theta^{t+1}$ in an online fashion. 
In \cite{mwnet}, the authors use a single layer (with 100 hidden nodes) feed-forward networks as the weighting network. Note, the input to the weight network is the scalar loss value $\ell_{\ce}^{i,\text{train}}\big(\bw(\Theta)\big)$ and the output of the weight network is the scalar sample weight value. The choice of weighting network is effective since a single hidden layer MLP is a universal approximator for any continuous smooth functions. Moreover, the weighting function is mostly a monotonically decreasing function under label noise.

\paragraph{Meta-Weight-Net with noisy meta samples.}
We can now analyze MW-Net and show how it can work without access to clean meta labels. We first discuss the gradient descent direction of the weighting network with clean meta samples. We denote the term 
\begin{align}
   \CG(\hat{\bw}) = \frac{1}{m}\sum_{j=1}^m\frac{\partial\ell^{j,\text{meta}}(\hat{\bw})}{\partial \hat{\bw}}\Big|_{\hat{\bw}^t}
    \label{eq:meta-gradient}
\end{align}
as the average \emph{meta-gradient} for $\bw$ on the meta dataset.
Taking derivative of $\ell^{j,\text{meta}}(\hat{\bw}^t(\Theta))$ w.r.t $\Theta$ in Eq.~\ref{eq:meta-update}, the $\Theta$ gradient descent direction is,%
\begin{align}
&\nabla_{\Theta}\sum_{j=1}^m \ell^{j, \meta}\big(\hat{\bw}^t (\Theta)\big)\label{eq:grad}\\
&=\!-\frac{\alpha}{n}\!\sum_{i=1}^n\! \Big(\CG(\hat{\bw})^{\intercal} \frac{\partial \ell_{\ce}^{i,\train}(\bw)}{\partial \bw}\Big|_{\bw^t}\!\Big) \frac{\partial \CW(\ell_{\ce}^{i,\train}(\bw^t);\Theta)}{\partial \Theta}\Big|_{\Theta^t}\nonumber
\end{align}
where $m$ and $n$ is the minibatch size for meta dataset and training dataset respectively. We can understand this update direction as a sum of weighted gradient updates for each training samples.
The term $\frac{\partial \CW(\ell_{\ce}^{i,\train}(\bw^t);\Theta)}{\partial \Theta}$ represents the gradient direction for $\Theta$ on training sample $i$. 
The weight for training sample $i$ is computed using $\CG(\hat{\bw})^{\intercal} \frac{\partial \ell_{\ce}^{i,\train}(\bw)}{\partial \bw}\Big|_{\bw^t}$; thus, the weighting network effectively puts more weights for training samples having a gradient direction similar to the average meta-gradient direction $\CG$. 

For corrupted meta samples, we denote the meta loss function as,
$\CL^{\text{noisy-meta}}(\bw^{\ast}(\Theta))\triangleq\frac{1}{M}\sum_{j=1}^M \ell^{j,\text{noisy-meta}}(\bw^{\ast}(\Theta))$,
where we use $\ell^{j,\text{noisy-meta}}$ to denote the loss on the corrupted $j^{\text{th}}$ meta sample (and  $\ell^{j,\text{meta}}$ to denote the loss for the $j^{th}$ clean meta sample).
Therefore, the key idea is that  if we need to use noisy meta dataset (training dataset can be corrupted arbitrarily and independently), we need to ensure
\begin{align}
\sum_{j=1}^m \frac{\partial \ell^{j,\text{meta}}(\hat{\bw})}{\partial \hat{\bw}}\approx C\sum_{j=1}^m \frac{\partial \ell^{j,\text{noisy-meta}}(\hat{\bw})}{\partial \hat{\bw}}.
\label{eq:same-grad}
\end{align}
$C$ can be any positive constant; we only care about gradient direction in SGD. Note in Eq.~\ref{eq:grad}, only the average meta-gradient ($\CG(\hat{\bw})$) term changes if we use noisy meta samples; gradient for training samples remains the same.

We show that when the outer meta loss function is MAE  (thus $\ell^{j,\text{meta}}$ is $\ell_{\mae}^{j,\text{meta}}$), the noisy (under uniform noise) average meta-gradients remain the same as the average meta gradients on clean meta dataset. 
In \cite{mwnet}, the CE loss function is used  with clean meta samples ($\ell^{j,\text{meta}}$ is $\ell_{\ce}^{j,\text{meta}}$).
Note, the classifier network still uses cross-entropy loss (we still use $\ell_{\ce}^{i,\train}$); we need to maintain average meta-gradient direction for meta samples only. We denote the classifier output (softmax probability vector) $f(\bx_j,\bw)$ as $\mathbf{u}_j\in \Delta^{K-1}$ where $\Delta^{K-1}$ is the $K$-1 dimensional simplex; MAE and CE losses are computed as, $\ell_{MAE}(\by_j,\bu_j)=\sum_k |\bu_{j,k}-\by_{j,k}|$ and $\ell_{CE}(\by_j,\bu_j)=-\sum_k \by_{j,k}\log \bu_{j,k}$.
Note that MAE loss $\ell_{\mae}$ is bounded and has a symmetric property \cite{ghosh2017} as:
\begin{equation}
    \sum_{c=1}^K \ell(c, f(\bx,\bw))= \text{constant, }\forall \bx, \text{ and }\forall \bw.
    \label{eq:sym}
\end{equation}
\begin{theorem}
Let the meta samples are corrupted with uniform noise with a rate of $\eta<1$. Suppose the weighting  network minimize loss function $\ell$ on the meta samples and the loss function $\ell$ satisfies symmetric property in Eq.~\ref{eq:sym}.
Then the expected meta-gradient on corrupted meta samples (for updating the weighting network in Eq.~\ref{eq:grad}) is exactly the same as (up to a proportional constant) the meta-gradient on clean meta samples irrespective of the noise model in the training dataset and the loss function in the classifier network.
\label{thm:main}
\end{theorem}
\begin{proof}
We can compute the expected gradient for $\bw$ on a batch of samples from the noisy meta dataset under uniform label noise (with the rate $\eta$) as:
 \resizebox{.99\linewidth}{!}{
  \begin{minipage}{\linewidth}
\begin{align*}
&\BE\Big[\sum_{j=1}^m \frac{\partial \ell^{j,\text{noisy-meta}}(\hat{\bw})}{\partial \hat{\bw}}\Big]\\
&\quad=\sum_{j=1}^m\Big[ (1-\eta) \frac{\partial \ell(\by_j, f(\bx_j,\hat{\bw}))}{\partial \hat{\bw}}+ \frac{\eta}{K}\sum_{c=1}^K\frac{\partial \ell(c, f(\bx_j,\hat{\bw}))}{\partial \hat{\bw}} \Big]\\
&\quad= \sum_{j=1}^m\Big[ (1-\eta) \frac{\partial \ell(\by_j, f(\bx_j,\hat{\bw}))}{\partial \hat{\bw}}+
\frac{\eta}{K}\frac{\partial}{\partial \hat{\bw}}\sum_{c=1}^K{ \ell(c, f(\bx_j,\hat{\bw}))} \Big]\\
&\quad= (1-\eta)\sum_{j=1}^m\Big[ \frac{\partial \ell(\by_j, f(\bx_j,\hat{\bw}))}{\partial \hat{\bw}}\Big]= C \sum_{j=1}^m \frac{\partial \ell^{j,\text{meta}}(\hat{\bw})}{\partial \hat{\bw}}
\end{align*}
  \end{minipage}
}
 since $\frac{\partial}{\partial \hat{\bw}}\sum_{c=1}^K{ \ell(c, f(\bx_i,\hat{\bw}))} =0$ when $\ell$ is a symmetric loss (such as MAE). Thus, under uniform label noise on the meta dataset, the expected meta-gradient on the noisy meta samples remains the same as the meta-gradient on the clean meta dataset when the meta loss function is symmetric. 
 \end{proof}
 \vspace{-0.1in}
 Theorem~\ref{thm:main} shows that symmetric losses on the corrupted meta samples for optimizing the weighting network has the same expected meta-gradient direction as the clean samples irrespective of the classifier network loss function and arbitrarily corrupted training datasets. Moreover, we can show convergence of the weighting network with finite mini batch size of the corrupted meta datasets under some mild condition. Similar to \cite{mwnet}, we make the following assumptions.
 \begin{enumerate}[leftmargin=*]
 \vspace{-0.05in}
 \item The meta loss $\ell$ and the classifier network loss $\ell^{\cdot,\text{train}}$ is Lipschitz smooth with constant L and have $\rho$-bounded gradients. 
 \vspace{-0.08in}
 \item The weighting function $\CW(\cdot)$ has bounded gradient and twice differential with bounded Hessian. %
 \vspace{-0.08in}
 \item The classifier network learning rate $\alpha_t$ satisfies $\alpha_t =\min\{1, \frac{k}{T}\}$ for some $k<T,\ k>0$. The learning rate of the weighting network $\beta_t$ satisfies $\beta_t =\min\{\frac{1}{L},\frac{b}{\hat{\sigma}\sqrt{T}}\} $ for some $b>0$ such that $\frac{\hat{\sigma}\sqrt{T}}{b}\geq L$ where $\hat{\sigma}^2$ is the variance of drawing a minibatch (possibly corrupted with  uniform noise).
 \end{enumerate}

\begin{theorem} \label{thm:converge}
Under assumptions 1-3, the weighting network converges for both clean and (uniformly) corrupted meta datasets when the meta loss satisfies symmetric condition:
\begin{align*}
&\min_{0\leq t \leq T} \mathbb{E}[ \|\nabla \mathcal{L}^{\meta}(\Theta^{t})\|_2^2] \leq \mathcal{O}(\frac{ \sigma}{\sqrt{T}}), \text{ and}\\
&\min_{0\leq t \leq T} \mathbb{E}[ \|\nabla \mathcal{L}^{\text{noisy-meta}}(\Theta^{t})\|_2^2] \leq \mathcal{O}\Big({\frac{{ \hat{\sigma}}}{(1-\eta)\sqrt{T}}}\Big),
\end{align*}
where $\sigma^2$ is the variance of drawing uniformly mini-batch sample at random, $\hat{\sigma}^2=\sigma^2+\frac{2\eta\rho^2}{m}$ is the variance adjusted with the uniformly corrupted meta samples,  $\eta$ is the uniform noise rate and $m$ is the minibatch size of the meta dataset\footnote{Derivation is available in supplementary material.}. 
	\end{theorem}
\begin{remark}
Using the convergence rate in Theorem~\ref{thm:converge}, we would expect to approximate the meta-gradient under noisy labels to be similar to the gradient under clean labels with a relatively large mini-batch size \cite{robbins1951stochastic}. 
Thus, we can use corrupted meta dataset as long as the meta loss function is symmetric \cite{ghosh2017}. The classifier network can use CE loss; the weighting network only needs to use symmetric loss for computing meta-gradients from the corrupted meta samples.  MAE loss is difficult to optimize for deep networks \cite{generalized-ce}; however, we note that the weighting network is simple with a single input (the loss value) and a single layer with 100 hidden nodes as used in \cite{mwnet}. We find that the weighting network with MAE loss for the meta samples is amiable to SGD; we also find MAE loss to be better than cross-entropy for the weighting network under varying noise rates and noise models.  Although, MAE loss can provide a guarantee for the meta dataset corrupted with uniform label noise; the training datasets do not require any such condition; we can potentially handle training datasets with \emph{instance-dependent label noise} also. We experimentally found that MAE loss significantly outperforms CE loss in the weighting network for uniform label noise. Nevertheless, we also observe that CE loss in the weighting network performs relatively well for noisy meta datasets.  We believe as the weighting network is a shallow  network with single input and output and equipped with $l_2$ regularizer (in the form of weight decay), $\frac{\eta}{K} \BE[\frac{\partial \sum_{c=1}^k{ \ell_{\ce}(c, f(\bx,\hat{\bw}))}}{\partial \bw}]$ is relatively small compared to $(1-\eta)\BE[ \frac{\partial \ell_{\ce}(\by, f(\bx,\hat{\bw}))}{\partial \hat{\bw}} ]$ even for CE loss.
\end{remark}

Under the flip noise, we randomly corrupt samples from a single class to a single different random class with a probability of $\eta$. Under the flip$2$ noise, we corrupt uniformly between random two fixed classes. Under a flip$\{K-1\}$ noise model, labels are corrupted to $K-1$ other classes, effectively becoming uniform noise. We consider the flip noise first; we denote the random corrupted class labels as $\CT(\by),\ \CT(\by)\neq \by$  where  samples with true labels $\by$ are corrupted as $\CT(\by)$  with probability $\eta$. However, once $\CT(\by)$ is drawn randomly as part of the noise model, it is deterministic. Hence, the expectation of gradient on  noisy meta samples do not remain the same as the gradient on the clean meta samples for symmetric losses:
\begin{equation*}
\resizebox{0.48\textwidth}{!}{$
\frac{\partial \CL^{\text{noisy-meta}}}{\partial \hat{\bw}} \!\!=\!\!\sum\limits_{j=1}^m\!\!\Big[ \!(\!1\!\!-\!\!\eta\!)\!\frac{\partial \ell(\!\by_j,f(\bx_j,\hat{\bw})\!)}{\partial \hat{\bw}}\!\!+\!\!\eta \!\frac{\partial \ell(\!\CT(\by_j), f(\bx_j,\hat{\bw})\!)}{\partial \hat{\bw}} \!\Big].
$}
\end{equation*}
Hence the theoretical guarantees in Eq. \ref{eq:same-grad} only holds for uniform noise model. 
Nevertheless, taking expectation w.r.t. to $\CT(\by)$, we get the same gradient as that of clean ones  since $\BE_{\CT(\by)} \Big[\frac{\partial \CL^{\text{noisy-meta}}}{\partial \hat{\bw}} \Big] =C \sum_{j=1}^m \frac{\partial \ell^{j,\text{meta}} (\hat{\bw})}{\partial \hat{\bw}}$. 
Thus, as we flip to more classes, we would expect more uniform label noise behavior resulting in the theoretical robustness of the weighting network using robust losses. However,  we found that when flip noise is relatively small, we can still use MW-Net with noisy meta samples. Moreover, we experimentally observe no significant gains for using clean meta samples even for flip noise (where labels are corrupted to a single other class). We extensively experiment the robustness of the MW-Net under the flip and flip2 noise model in Section~\ref{sec:exp}.

%% file: exp.tex
\begin{table*}[t]\centering
    \scalebox{0.95}{
    \begin{tabular}{c ccc| ccc}\toprule
    
 \multirow{ 3}{*}{Model}    &  \multicolumn{3}{c|}{Uniform Noise Rate} &  \multicolumn{3}{c}{Flip2 Noise Rate} \\
    \cline{2-7}
             & $0\%$ & $40\%$ & $60\%$   & $0\%$ & $20\%$ & $40\%$  \\
    \cline{2-7}        
            &   \multicolumn{6}{c}{Dataset\quad CIFAR-10}\\
    \midrule
    \toprule
            L2RW$^{\ast}$ &$92.38\pm 0.10$ & $86.92\pm 0.19$ & $82.24\pm 0.36$ & $89.25\pm 0.37$ & $87.86\pm 0.36$ & $85.66\pm 0.51$ \\
            GLC$^{\ast}$&$94.30\pm 0.19$ & $88.28\pm 0.03$ & $83.49\pm 0.24$ & $91.02\pm 0.20$ & $89.68\pm 0.33$ & $\bf{88.92\pm 0.24}$\\
            MW-Net$^{\ast}$ & \textbf{95.15$ \pm $0.13} & \textbf{90.35$ \pm $0.21} & \textbf{86.52$ \pm $0.09} & \textbf{92.35$ \pm $0.29} & \bf{90.28$ \pm $0.27} & 87.04$ \pm $0.89\\
            
            \toprule
            BaseModel & $95.60\pm 0.22$ & $68.07 \pm 1.23$ & $55.12 \pm 3.03$ & $92.89\pm 0.32$ & $76.83\pm 2.30$ & $70.77\pm 2.31$ \\
            Reed-Hard & $94.38 \pm 0.14 $ & $81.26\pm 0.51 $ & $73.53 \pm 1.54 $ & $92.31\pm 0.25$ & $88.28\pm 0.36$ & $81.06\pm 0.76$ \\
            S-Model & $83.79\pm 0.11$ & $79.58 \pm 0.33$ & $-$ & $83.61\pm 0.13$ & $79.25\pm 0.30$ & $75.73\pm 0.32$ \\
            Self-paced & $90.81 \pm 0.34$ & $86.41\pm 0.29$ & $53.10\pm 1.78$ & $88.52\pm 0.21$ & $87.03\pm 0.34$ & $81.63\pm 0.63$ \\
            Focal Loss & $95.70\pm 0.15$ & $75.96\pm 1.31$ & $51.87\pm 1.19$ & $93.03\pm 0.16$ & $86.45\pm 0.19$ & $80.45\pm 0.97$ \\
            Co-teaching &$88.67\pm 0.25$ & $74.81\pm 0.34$ & $ 73.06 \pm 0.25$ & $89.87\pm 0.10$ & $82.83\pm 0.85$ & $75.41\pm 0.21$\\
            D2L& $94.64\pm0.33$ & $85.60\pm 0.13$ & $68.02 \pm 0.41$ & $92.02\pm 0.14$ & $87.66\pm 0.40$ & $83.89\pm 0.46$  \\
            Fine-tuning &$\bf{95.65\pm 0.15 }$ & $80.47\pm 0.27 $ & $78.75\pm 2.40$ & $\bf{93.23\pm 0.23}$ & $82.47\pm 3.64$ & $74.07\pm 1.56$ \\
            MentorNet&$94.35\pm 0.42$ & $ 87.33\pm 0.22$ & $82.80\pm 1.35$ & $92.13\pm 0.30$ & $86.36\pm 0.31$ & $81.76\pm 0.28$\\
            MNW-Net & 94.98$ \pm $0.1& 88.8$ \pm $0.44 & 84.26$ \pm $0.18 & 92.45$ \pm $0.11 & \textbf{90.67$ \pm $0.18} & \textbf{87.92$ \pm $0.44}\\
            RMNW-Net &95.23$ \pm $0.13 & \bf{90.8$ \pm $0.23}& \bf{86.31$ \pm $0.28}& {92.66$ \pm $0.13} & 89.78$ \pm $0.24 &85.38$ \pm $0.37 \\
            \midrule
            \toprule
            &  \multicolumn{6}{c}{Dataset\quad CIFAR-100}\\ \cline{1-7}
            L2RW$^{\ast}$&$72.99\pm 0.58$ & $60.79\pm 0.91$ & $48.15\pm 0.34$ & $64.11\pm 1.09$ & $57.47\pm 1.16$ & $50.98\pm 1.55$\\
            GLC$^{\ast}$&$73.75\pm 0.51$ & $61.31\pm 0.22$ & $50.81\pm 1.00$ & $65.42\pm 0.23$ & $63.07\pm 0.53$ & $\bf{62.22\pm 0.62}$\\
            MW-Net$^{\ast}$ & \bf{78.06$ \pm $0.14} & \bf{70.39$ \pm $0.16} & \bf{63.29$ \pm $0.23} & \bf{68.92$ \pm $0.35} & \bf{64.01$ \pm $0.37} & 57.38$ \pm $0.48\\
            \midrule
            BaseModel &$79.95\pm 1.26$ & $51.11\pm 0.42$ & $30.92\pm 0.33$ & $70.50\pm 0.12 $ & $50.86\pm 0.27$ & $43.01\pm 1.16$\\
            Reed-Hard&$64.45\pm 1.02 $ & $51.27\pm 1.18$ & $26.95\pm 0.98$ & $69.02\pm 0.32$ & $60.27\pm 0.76$ & $50.40\pm 1.01$\\
            S-Model &$52.86\pm 0.99$ & $42.12\pm 0.99$ & $-$ &$51.46\pm 0.20$ & $45.45\pm 0.25$ & $43.81\pm 0.15$\\
            Self-paced &$59.79\pm 0.46$ & $46.31\pm 2.45$ & $19.08\pm 0.57$ & $67.55\pm0.27$ & $63.63\pm 0.30$ & $53.51\pm 0.53$\\
            Focal Loss &$81.04\pm 0.24$ & $ 51.19\pm 0.46$ & $27.70\pm 3.77$ & $70.02\pm 0.53 $ & $61.87\pm 0.30$ & $54.13\pm 0.40$\\
            Co-teaching& $61.80\pm 0.25$ & $46.20\pm 0.15$ & $35.67\pm 1.25$ & $63.31\pm 0.05$ & $54.13\pm 0.55$ & $44.85\pm 0.81$\\
            D2L&$66.17\pm 1.42$ & $52.10\pm 0.97$ & $41.11\pm 0.30$ & $68.11\pm 0.26$ & $63.48\pm 0.53$ & $51.83\pm 0.33$\\
            Fine-tuning & $\bf{80.88\pm 0.21}$ & $52.49\pm 0.74$ & $38.16\pm 0.38$ & $\bf{70.72\pm 0.22}$ & $56.98\pm 0.50$ & $46.37\pm 0.25$\\
            MentorNet&$73.26\pm 1.23$ & $61.39\pm 3.99$ & $36.87\pm 1.47$ & $70.24\pm 0.21$ & $61.97\pm 0.47$ & $52.66\pm 0.56$\\
            MNW-Net & 77.97$ \pm $0.12 &62.95$ \pm $0.38& 55.75$ \pm $0.56 & 68.83$ \pm $0.31 & 63.69$ \pm $0.48 & 56.46$ \pm $0.28\\
            RMNW-Net & 78.27$ \pm $0.17 & $\bf{70.76 \pm 0.20}$ & \textbf{63.79$ \pm $0.39} & 69.69$ \pm $0.21 & \textbf{65.46$ \pm $0.51} &\textbf{57.41$ \pm $0.51}  \\
      \bottomrule
        \end{tabular}
    }%
    \vspace{0.03cm}
    \caption{Classification accuracy on the clean test set of CIFAR-10/CIFAR-100 dataset under uniform noise and flip2 noise (using WRN-28-10 architecture for uniform noise and ResNet-32 for flip2 noise \cite{mwnet}). L2RW$^{\ast}$,  GLC$^{\ast}$, and MW-Net$^{\ast}$ use 1000 clean meta samples; all other models use 1000 corrupted meta samples. Best models from each group are \bf{bold}.
    }
    \label{tab:main}
\end{table*}

\begin{table*}[t]\centering
    \scalebox{0.9}{
    \begin{tabular}{c ccc| ccc}\toprule
    \multirow{ 3}{*}{Model}  &   \multicolumn{3}{c|}{CIFAR-10}  &  \multicolumn{3}{c}{CIFAR-100} \\
    \cline{2-7}
    & \multicolumn{6}{c}{Flip Noise Rate}\\
     & $0\%$ & $15\%$ & $30\%$   & $0\%$ & $15\%$ & $30\%$  \\\toprule
            MWNet$^{\ast}$ & \textbf{95.15 $\pm$ 0.13} & \textbf{93.73$ \pm $0.2} & \textbf{91.01$ \pm $0.31} & \textbf{78.06$\pm$ 0.14}& \textbf{73.98$ \pm $0.32}& \textbf{66.84$ \pm $0.25}\\\midrule
            MNW-Net & 94.98 $\pm$ 0.1 & \textbf{93.75$ \pm $0.15}& \textbf{92.01$ \pm $0.23} & 77.97 $\pm$ 0.12 & 73.79$ \pm $0.26 &66.38$ \pm $0.51\\
            RMNW-Net & \textbf{95.23 $\pm$0.13} &93.17$ \pm $0.12 & 88.69$ \pm $0.53 & \textbf{78.27$ \pm $0.17} &\textbf{74.16$ \pm $0.25} & \textbf{66.72$\pm$0.37} \\
      \bottomrule
        \end{tabular}
    }
    \vspace{0.03cm}
        \caption{Classification accuracy on clean test set of CIFAR-10/CIFAR-100 dataset under flip noise (using WRN-28-10 architecture). Best models from each group are \bf{bold}.}
    \label{tab:flip}
\end{table*}

\begin{table*}[t]\centering
    \scalebox{0.85}{
    \begin{tabular}{c cc|cc|cc||cc|cc|cc}\toprule
    \multirow{ 3}{*}{Model/ Noise Rate}  & \multicolumn{6}{c||}{CIFAR-10} & \multicolumn{6}{c}{CIFAR-100} \\
    \cline{2-13}
    &  \multicolumn{2}{c|}{Uniform} & \multicolumn{2}{c|}{Flip2} & \multicolumn{2}{c||}{Flip} & \multicolumn{2}{c}{Uniform} & \multicolumn{2}{c|}{Flip2} & \multicolumn{2}{c}{Flip} \\
      & 40\% & 60\% & 20\% & 40\% & 15\% & 30\%    & 40\% & 60\% & 20\% & 40\% & 15\% & 30\%  \\\toprule
            MW-Net$^{\ast}$ & \textit{0.9714} & \textit{0.9631} & \textit{0.9487} & \textit{0.9647} & \textbf{0.9857} & \textbf{0.9721} & \textit{0.9644} & \textit{0.9581} & \textit{0.8561} & \textit{0.5962} & \textit{0.7948} & \textit{0.5145}\\
            MNW-Net & 0.9346 & 0.8874 & 0.2531 & 0.5137 & 0.2692 & 0.4901 & 0.9587 & 0.9390 & 0.4946 & 0.5647 & 0.3407 & 0.4698\\
            RMNW-Net & \textbf{0.9848} & \textbf{0.9879} & \textbf{0.9692} & \textbf{0.9694} & \textit{0.9807} & \textit{0.9614} & \textbf{0.9749} & \textbf{0.9800} & \textbf{0.9445} & \textbf{0.9001} & \textbf{0.9322} & \textbf{0.7469}    \\
      \bottomrule
        \end{tabular}
    }
    \vspace{0.03cm}
    \caption{AUC for detecting noisy training samples. Best models are {\bf bold}, second best models are {\it italic}.}
    \label{tab:auc}
    \vspace{-0.2in}
\end{table*}

\section{Experiments}
\label{sec:exp}
We perform a series of experiments to evaluate the robustness of the weighting network under noisy meta samples and compare our approach with competing methods.
We follow the experimental setup in for fair comparison  \cite{mwnet}. 

\textbf{Dataset}
We use two benchmark datasets CIFAR-10 and CIFAR-100 for comparing robust learning methods. We use 1000 meta samples for the weighting network. However, contrary to \cite{mwnet}, we do not require clean samples for the meta dataset. Thus, we also experiment with corrupted meta samples. In our results, we differentiate between models using clean (denoted with {\it asterisk}$^{\ast}$) and noisy meta samples.

\textbf{Noise Rate} We apply the uniform noise model with rates $0$, $0.4$, and $0.6$ and the flip2 noise model with rates $0$, $0.2$, $0.4$. Furthermore, we also compare against conditions under heavily corrupted training samples with a $0.7$ uniform label noise rate and a $0.5$ flip2 label noise rate. 
Moreover, we also apply flip noise with a rate of $0.15$ and $0.30$ and compare with MW-Net\footnote{The flip2 noise model is denoted as flip noise in \cite{mwnet}}. Note that the flip noise  is more challenging than flip2 and uniform noise model; flip/flip2/uniform noise with rate $\eta\geq0.50 / 0.67 / 1.0$  would result in noisy class to be the majority making learning infeasible.

\textbf{Baseline methods}
Our analysis shows the weighting network optimized with MAE loss on corrupted meta samples has the same expected gradient direction as of clean meta samples.  We denote our model, utilizing corrupted samples and using a robust MAE loss for the weighting network, as Robust-Meta-Noisy-Weight-Network (referred to as RMNW-Net).
We denote the MW-Net model utilizing corrupted meta samples as Meta-Noisy-Weight-Network (referred to as MNW-Net); thus, the MNW-Net model trains the weighting network on the noisy meta dataset using cross-entropy loss as the meta loss function.  
We distinguish models using noisy meta samples and clean meta samples differently. L2RW \cite{l2rw}, GLC \cite{glc}, MW-Net uses clean meta samples and they are denoted as L2RW$^{\ast}$, GLC$^{\ast}$, and MW-Net$^{\ast}$ respectively. Other competing methods include BaseModel, trained on corrupted training data, and several robust learning methods such as Reed \cite{reed}, S-Model \cite{s-model}, SPL \cite{spl},  Focal Loss \cite{focal}, Co-teaching \cite{co-teaching}, D2L \cite{d2l}, MentorNet \cite{mentornet}. We train MW-Net, MNW-Net, RMNW-Net from scratch for all the datasets and the noise models; for the baseline models, we reuse results from \cite{mwnet}. We obtained slightly better performance for MW-Net and its variants with a smaller learning rate than reported in \cite{mwnet}.

\textbf{Network architecture and optimization}
For uniform noise, we use Wide ResNet-28-10 (WRN-28-10) \cite{wide}, and for flip2 noise, we use ResNet-32 \cite{resnet}, following \cite{mwnet}. 
For the flip noise model, we use WRN-28-10 architecture \footnote{We also experiment with flip2 noise with WRN-28-10 to show indifference to architecture for MW-Net model and its variants which we detail in supplementary material.}. The weighting network is a single layer neural network with 100 hidden nodes and ReLU activations. 
We train for MW-Net and its variant for $120$ epochs using SGD with momentum $0.9$, weight-decay $5\times 10^{-4}$, and an initial learning rate of $0.05$.
We use an optimizer scheduler to divide the learning rate by 10 after 36 and 38 epoch (for uniform, flip, and flip2 noise with WRN-28-10) and after 40 and 50 epoch (for flip2 noise with ResNet-32), following \cite{mwnet}. The learning rate for the weighting network is fixed to $10^{-3}$ with weight decay $5e-4$. We do not perform hyper-parameter optimization following \cite{mwnet}. We use a batch size of 100 for both the training samples and the meta samples.  We repeat experiments with five different seeds for corrupting samples with label noise and initializing the classifier networks. 

\subsection{Results}
Table~\ref{tab:main} lists the average accuracy and standard deviation (std) over five runs for all models on CIFAR-10/100 dataset corrupted with uniform and flip2 noise model. Under $40\%$ and $60\%$ uniform noise, performance of MNW-Net drops by $6\%$ and $12\%$ on CIFAR-10 dataset and drops by $20\%$ and $30\%$ on CIFAR-100 dataset compared to the no noise case. Other baseline models using corrupted meta samples performs worse than MNW-Net. However, using the clean meta samples, performance of MW-Net$^{\ast}$ only drops by $4\%$ and $9\%$ on CIFAR-10 dataset and drops by $10\%$ and $20\%$ on CIFAR-100 dataset. Only our proposed approach RMNW-Net, even without access to clean meta samples, retains performance similar to MW-Net$^{\ast}$ on both dataset.

Under flip2 noise, we observe MNW-Net performs slightly better than RMNW-Net on CIFAR-10 dataset while RMNW-Net performs better on CIFAR-100 dataset. In CIFAR-10/CIFAR-100 dataset, MNW-Net/RMNW-Net performs better than all methods using corrupted meta samples. Moreover, we also note that both MNW-Net and RMNW-Net performs similar to MW-Net without access to the clean meta samples for the flip2 noise model. This observation suggests that even for flip2 noise, we may not need access to clean meta samples. Under clean datasets ($0\%$ noise rate), we observe that RMNW-Net performs marginally better than MW-Net$^{\ast}$ for both WRN-28-10 architecture and ResNet-32 architectures. Improved performance of RMNW-Net compared to MW-Net$^{\ast}$ on clean datasets may suggest MAE loss is suitable for the weighting network for achieving better generalization ability; we leave such studies for future works.

Table \ref{tab:flip} lists the average accuracy and std for MW-Net$^{\ast}$, MNW-Net, and RMNW-Net for the flip noise model. Similar to flip2 noise, we observe that under flip noise, MNW-Net performs better on CIFAR-10 dataset whereas RMNW-Net performs better on CIFAR-100 dataset. Moreover, we do not observe any significant drop in performance when using corrupted meta samples compared to MW-Net$^{\ast}$. Surprisingly, in some cases ($30\%$ flip noise on CIFAR-10 dataset), we observe, MNW-Net performs better than MW-Net$^{\ast}$; a slightly larger std number might be a reason behind this.

\textbf{Robustness under heavy label noise}
Figure~\ref{fig:heavy-noise} shows performances of MW-Net$^{\ast}$, MNW-Net, and RMNW-Net under heavy uniform ($\eta =0.7$) and flip2 ($\eta=0.5$) noise.  We observe that for $70\%$ uniform noise, RMNW-Net performs better than MNW-Net by $\sim15\%$ and $\sim 1.5\%$ on CIFAR-10 and CIFAR-100 dataset respectively.  Under flip2 noise, MNW-Net performs better on CIFAR-10 dataset and RMNW-Net performs better on CIFAR-100 dataset. These observations along with our analytical results suggest RMNW-Net performs the best on when meta samples are corrupted with uniform noise and performs on par with MW-Net$^{\ast}$ in almost all cases. 

\textbf{Corrupted sample detection}
We also study whether the weighting network can separate the noisy training samples from the clean training samples. We take the output of the weighting network as the probability of being clean samples and compute the AUC metric on the corrupted training datasets. Table~\ref{tab:auc} lists the AUC numbers on predicting the clean samples for MW-Net$^{\ast}$, MNW-Net, and RMNW-Net for different noise models and noise rates.  We run for 120 epochs and test the performances on corrupted training datasets while training; in Table~\ref{tab:auc}, we list the best AUC numbers obtained while training the weighting network and the classifier network. MW-Net$^{\ast}$ performs better than MNW-Net as expected due to clean meta samples. However, surprisingly, RMNW-Net, with corrupted meta samples, performs better than MW-Net$^{\ast}$ with clean meta samples. 
We also observe MW-Net$^{\ast}$ performs well (in terms of AUC numbers) under uniform noise as also observed in \cite{mwnet}; however, the performance deteriorates significantly for flip2 and flip noise model on CIFAR-100 dataset. Nevertheless, RMNW-Net retains strong performance for all cases. Surprisingly, although RMNW-Net has better performances in terms of AUC numbers, on the primary classification task, we observe MW-Net$^{\ast}$, MNW-Net, and RMNW-Net to perform similarly.
Further investigations reveal that for RMNW-Net, the importance weights for most samples are in a tighter range whereas, for MW-Net$^{\ast}$, the importance weights for most samples spread out over a larger range. A smaller gap in RMNW-Net might be the reason for not having significantly improved classification performance. We note that having a high AUC number on the corrupted sample detection task is not a prerequisite for having the best performance on the primary classification task. However, a high discrimination ability to separate the corrupted samples is indeed important for devising a more powerful robust model.
We suspect that RMNW-Net can filter out corrupt samples from the noisy training datasets using a small number of corrupted meta samples with high accuracy, though we leave
confirmation to future work.

	\begin{figure}[pt] \vspace{2mm}
		\begin{minipage}{0.49\textwidth}
			\centering
			\includegraphics[width=0.49\textwidth]{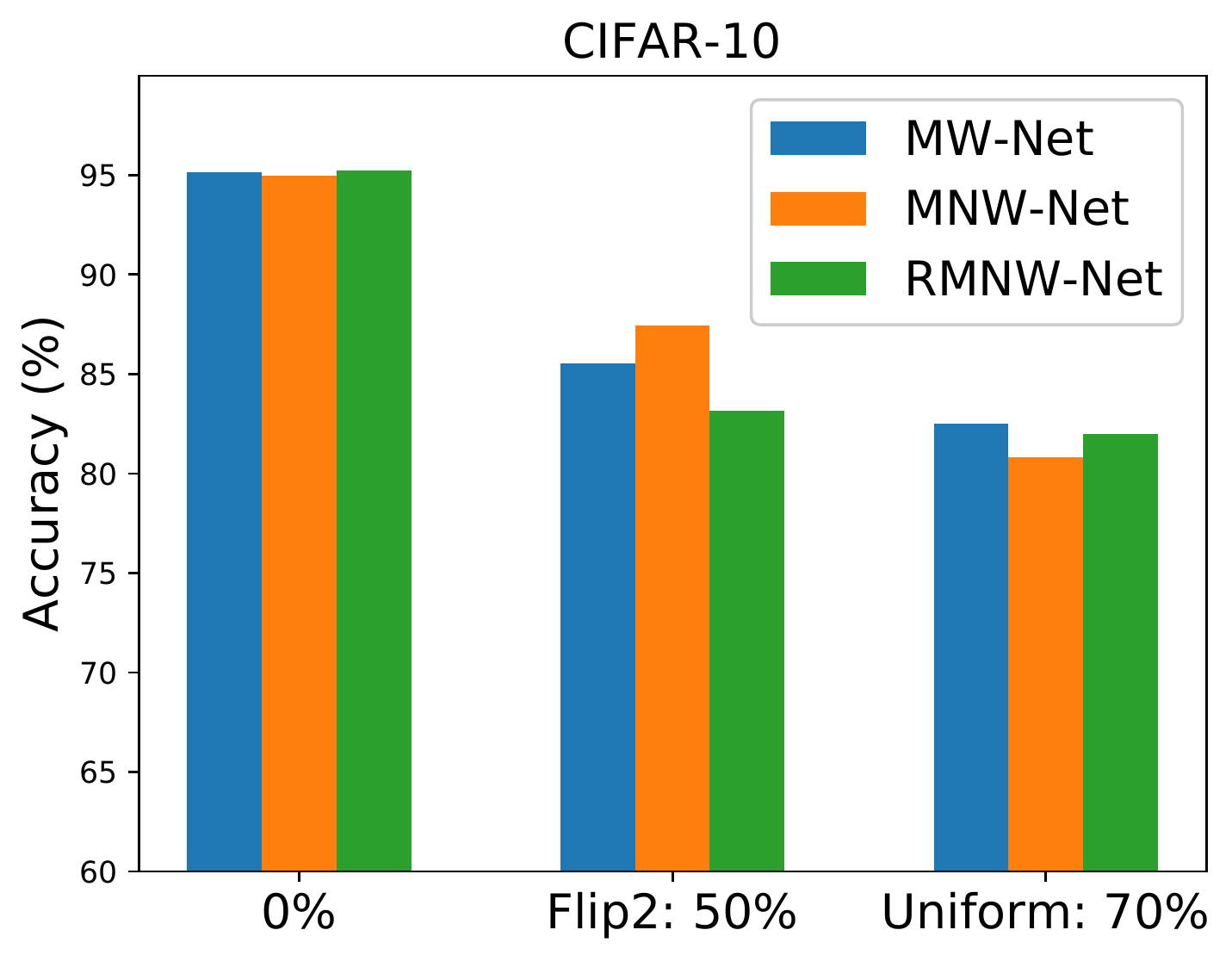}
			\includegraphics[width=0.49\textwidth]{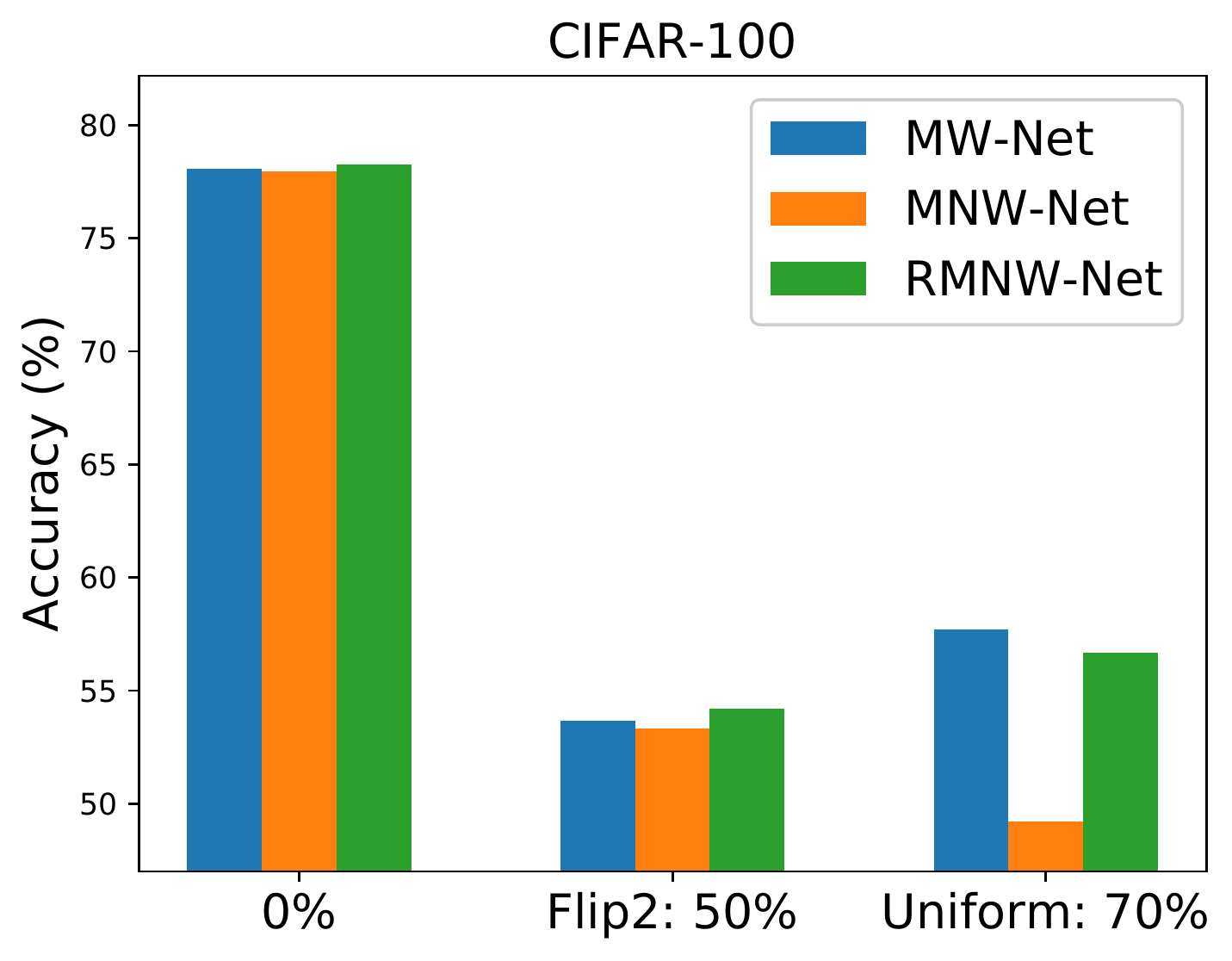} \vspace{0mm}
			\caption{Performance comparison for different models under heavy uniform and flip2 noise (using WRN-28-10 architecture)} \label{fig:heavy-noise}
		\end{minipage}\vspace{-0.2in}
	\end{figure}

%% file: related.tex
\vspace{-0.25cm}
\section{Related Works}
Machine learning models are highly vulnerable to label noise in the training datasets; label noise becomes more important for the DNN classifier with enormous memorization capability \cite{closer-memory}.  Commonly used convex potential losses \cite{bartlett} are not tolerant to label noise; theoretically, noise robustness of convex loss potentials have been analyzed in \cite{servedio}  and experimentally studied in \cite{nettleton}. A class of common approaches looks for models that are inherently tolerant to label noise; thus training with noisy labels should not affect the predictive performance on unbiased test datasets \cite{unhinged,ghosh2015, ghosh2017-dt}. For asymmetric label noise, an unbiased estimator can be constructed if the noise rates are known \cite{natarajan}. 
Computing the noise rates from corrupted datasets has been also studied under some strict assumptions \cite{scott}.
For binary classifiers, a sufficient condition has been derived for a loss functions to be tolerant to label noise in \cite{ghosh2015}; ramp loss, probit loss, unhinged loss satisfies this condition \cite{ghosh2015,unhinged,calibrated-surrogates}. For a multi-class DNN classifier, MAE loss is proved to be tolerant to label noise \cite{ghosh2017}. A generalized cross-entropy loss has been proposed that offers noise robustness with easier optimization procedure \cite{generalized-ce}. Along the line of robust MAE loss \cite{ghosh2017}, several new loss functions have been recently proposed \cite{normalized-loss,peer,ldmi, symmetric-ce,symmetric-labels,loss-factor,sound-noise,ce-ijcai}. Several methods have also focused on adding a noise transition matrix layer before computing the softmax probability distributions \cite{s-model,forward,glc,probabilistic-correction,xiao2015learning,safe-noise}. Among these, several methods additionally use a subset of clean gold samples for computing the transition matrix with some guarantees \cite{glc,forward,xiao2015learning}.
The core idea of the noise transition layer is to perform Expectation-Minimization (EM) algorithm to iteratively predict the noisy samples and learn the classifier networks  \cite{fisher-estimating}.

Another prominent approach for learning with label noise is to determine the clean samples in the corrupted training datasets. EM algorithms for label noise also fall under these settings where the expectation steps learn the probable clean samples.  The idea of choosing probable clean samples goes farther back; many heuristics have been proposed to determine the clean samples in the training procedure and remove the effect of those samples in earlier research \cite{identifying-mislabeled,identifying}. Recently, there has been a resurgence of methods using ideas along these lines \cite{s-model,Mandal_2020_WACV,co-teaching,joint,webly,co-teaching,mentornet,joint,veit2017learning,yuan2018iterative,focal,reed,masking}. 
A similar concept is introduced in curriculum learning where model learns from easy samples first, and then adapts to harder examples \cite{bengio2009curriculum,self-paced,webly,learning-distill}.
In Self-Paced learning (SPL), the model first optimizes easy samples (probably clean) samples \cite{spl, spl-div}.
In Mentornet \cite{mentornet}, the model consists of a student-teacher network where the teacher finds clean samples for the students. 
Co-teaching is another method with two neural networks where one network uses samples predicted as clean (with some threshold) using the other network \cite{co-teaching}.
Along this line, many recent methods for robust DNN learning use an ensemble (two or more networks) approach to mitigate the effect of label noise \cite{mentornet,yuan2018iterative,co-teaching, disagreement,active-bias}.
Although one network can be susceptible to noise, two (or more) networks can be possibly used to remove noisy samples. Some methods also augment an unbiased clean dataset; in \cite{yuan2018iterative}, authors use a different subset of the noisy dataset for multiple networks with the help of clean datasets whereas
in \cite{veit2017learning}, authors use some number of clean samples to learn a separate clean sample prediction networks.  Susceptibility of DNN under label noise is also due to memorization and overfitting capabilities. Based on the overfitting behavior, in \cite{d2l}, the authors proposed a dimensionality driven measure to determine the epochs where overfitting starts to occur. Along a similar line to \cite{d2l}, in \cite{o2u}, the learning rate is varied between overfitting and underfitting in the training procedure. However, while these methods perform well in practice, they do not guarantee robustness and often require hyper-parameter tuning and closely watching training losses.

Methods, identifying clean samples and retraining the classifiers, effectively reduce the importance weights for probable noisy samples. Thus, many proposed approaches learn the importance weights with or without the help of additional clean samples \cite{l2rw,focal,l2l-noise,mwnet,cleannet,hu2019noise,iterative-openset}. The core idea is to assign the importance weights based on loss values; for example, in Focal loss, the model gives more weight to high loss samples \cite{focal}. 
In \cite{reed}, the authors propose a regularized method to use samples that have labels similar to the model predictions; thus, the model implicitly gives low weights to high loss samples. Many recent methods  propose to learn the importance weight function by non-linear functions \cite{l2rw,mwnet}. 
Meta-learning methods are suitable for learning the weighting network to learn the classifier network \cite{maml,reptile,l2l-g2g}. The meta-learning framework learns a network that can adapt to many tasks with very few samples \cite{prototypical,reptile,ravi-few-shot}; model agnostic meta-learning, a powerful approach for meta-learning, optimizes an optimal initializer for many tasks \cite{maml}. 
In \cite{l2rw}, the authors use meta-learning methods L2R for learning to weight samples; similar to MW-Net, the model learns a weighting function using clean meta (validation) datasets. 
Along this line, several proposed methods  \cite{meta-noise-training,mwnet,l2l-noise,metacleaner,nlnl,combinatorial} use meta-learning frameworks.

%% file: con.tex
\vspace{-0.25cm}
\section{Discussions and Future Works}
An optimal sample weighting strategy is of utmost importance for learning robust deep neural networks; learning a weighting strategy alleviates the need for heuristics to assign weights to corrupted samples in the training datasets. However, learning the weighting network comes at a cost; previous research uses gold standard clean datasets for the weighting network. In this paper, we close the gap for learning the weighting network without access to clean meta samples. Theoretical results and extensive experimentation suggest that using robust loss functions in the weighting network can perform similar to the case with clean meta samples. Although MAE loss is difficult to optimize for deep networks, optimizing the weighting  (shallow) network with MAE loss is easy and effective. Albeit simple, we show that it is possible to learn robust classifiers by importance weighting without access to clean meta samples. 

Interestingly, we observe the weight network with robust losses has much more discernibility than the cross-entropy loss in separating the corrupted samples from the training datasets. Identifying the clean samples (instead of assigning importance weights) using the meta-network with robust losses is an intriguing avenue for future work. Although our proposed approach performs similar to MW-Net even without access to the clean meta samples for all the noise models, our theoretical results only hold for the uniform noise model on the meta dataset. Thus, the degree of robustness for MW-Net and RMNW-Net under the flip and flip2 noise model remains an open problem to tackle.

%% file: supp.tex
\begin{center}
\Large \bf 
Supplementary Material to\\
\large Do We Really Need Gold Samples for Sample Weighting under Label Noise?
\end{center}
\section{Convergence of Robust-Meta-Noisy-Weight-Network}
	We will detail the proof of convergence in Theorem~\ref{thm:converge} for our proposed method. Recall when we have clean  meta samples, the meta loss is computed as
	\begin{align}
	\mathcal{L}^{\meta}(\mathbf{w}^*(\Theta))=\frac{1}{M} \sum_{j=1}^M \ell^{\cdot,\text{meta}}(\mathbf{w}^*(\Theta)),
	\end{align}
	whereas for corrupted meta samples, the meta loss is computed as,
		\begin{align}
	\mathcal{L}^{\text{noisy-meta}}(\mathbf{w}^*(\Theta))=\frac{1}{M} \sum_{j=1}^M \ell^{\cdot,\text{noisy-meta}}(\mathbf{w}^*(\Theta)),
	\end{align}
where $\mathbf{w}^*$ is the optimal classifier network, and $\Theta$ is the parameter of the weighting network. 
 The classifier network is trained on the following objective,
	\begin{align}\label{eqob}
	\mathcal{L}^{\train}(\mathbf{w};\Theta) = \frac{1}{N} \sum_{i=1}^N \mathcal{W}(\ell_{\ce}^{\cdot, \train}(\mathbf{w});\mathbf{\Theta}) \ell_{\ce}^{\cdot,train}(\mathbf{w}).
	\end{align}

We use the following lemma from \cite{mwnet} for proving convergence results.
\begin{lemma}\label{lem:lemma1}
		Suppose the meta loss function is Lipschitz smooth with constant $L$, and $\mathcal{W}(\cdot)$ is differential with a $\delta$-bounded gradient and twice differential with its Hessian bounded by $\mathcal{B}$, and the loss function $\ell_{\ce}^{\cdot, \train}$ have $\rho$-bounded gradients with respect to training/meta data. Then the gradient of $\Theta$ with respect to meta loss is Lipschitz continuous.
	\end{lemma}
\begin{proof}
Detailed proof can be found in \cite{mwnet}.
\end{proof}

In Theorem~\ref{thm:main}, we showed that the expectation of the meta-gradient remains same for corrupted meta samples under uniform noise model. Next, we bound the variance of meta-gradient under uniformly corrupted meta samples. 
\begin{lemma}\label{lem:lemma2}
Suppose the meta loss function $\ell^{\cdot, \text{meta}}$ ($\ell^{\cdot, \text{noisy-meta}}$), satisfying symmetric condition in Eq.~\ref{eq:sym}, have $\rho$-bounded gradients with respect to meta data. Let the variance of drawing a minibatch (of $m$ samples) randomly is $\sigma^2$. Then the variance of the meta-gradients under uniformly corrupted meta samples (with rate $\eta$) is bounded by $\hat{\sigma}^2 = \sigma^2+\frac{2\eta \rho^2}{m}$. 
	\end{lemma}
	\begin{proof}
	Under clean meta samples, we have,
	\begin{align*}
	    \xi^{t} &= \nabla \mathcal{L}^{\meta}(\hat{\mathbf{w}}^{t}(\Theta^{t}))\big|_{\zeta_t}- \nabla \mathcal{L}^{\meta}(\hat{\mathbf{w}}^{t}(\Theta^{t}))\\
	    &= \nabla \mathcal{L}^{\meta}(\hat{\mathbf{w}}^{t}(\Theta^{t}))\big|_{\zeta_t}- \mathcal{K}(\hat{\bw}(\Theta^t))
	\end{align*}
	where the mini-batch  of size $m$, $\zeta_t$ is drawn uniformly from the entire clean meta data set and $\mathcal{K}(\hat{\bw}(\Theta^t))$ is the unbiased meta-gradient.  We also have
	    $\BE[||\xi^t||^2]  = \sigma^2$ for clean meta dataset.
	
	Under corrupted meta dataset, we have,
		\begin{align*}
	    \xi^{t} &= \nabla \mathcal{L}^{\text{noisy-meta}}(\hat{\mathbf{w}}^{t}(\Theta^{t}))\big|_{\zeta_t, \eta_t}- (1-\eta)\nabla\mathcal{L}^{meta}(\hat{\mathbf{w}}^{t}(\Theta^{t}))\\
	    &= \nabla \mathcal{L}^{\text{noisy-meta}}(\hat{\mathbf{w}}^{t}(\Theta^{t}))\big|_{\zeta_t,\eta_t}- (1-\eta)\mathcal{K}(\hat{\bw}(\Theta^t))
	\end{align*}
	since, we have shown in Theorem~\ref{thm:main}, meta-gradients of the corrupted meta dataset is upto a constant of the unbiased ones when the meta loss function satisfies symmetric condition.
	We note that for noisy meta dataset, the meta-gradient is,
	\[ \frac{1}{m}\sum_{j=1}^m\frac{\partial\ell^{j,\text{noisy-meta}}(\hat{\bw}(\Theta^{t}))}{\partial \hat{\bw}(\Theta^{t})}\Big|_{\hat{\bw}^t}\]
	We compute variance for a single meta sample and then use the variance of sum of independent random variable rule to compute the final variance. Note, for a single sample, $\BE[||\nabla \mathcal{L}^{\meta}(\hat{\mathbf{w}}^{t}(\Theta^{t}))\big|_{\zeta_t}- \mathcal{K}(\hat{\bw}(\Theta^t))||^2]=m\sigma^2$. Now we can compute the variance of corrupted meta-gradient when mini batch size is 1.
	\begin{align*}
	    &\BE_{\zeta_t, \eta_t}[||\xi^t||^2] = \BE_{\zeta_t, \eta_t}[|| \nabla \mathcal{L}^{\text{noisy-meta}}(\hat{\mathbf{w}}^{t}(\Theta^{t}))\big|_{\zeta_t,\eta_t}\\
	    &\quad-(1-\eta)\mathcal{K}(\hat{\bw}(\Theta^t))||^2]\\
	    &= \BE_{\zeta_t, \eta_t}\Big[ \mathcal{L}^{\text{noisy-meta}}(\hat{\mathbf{w}}^{t}(\Theta^{t}))\big|_{\zeta_t,\eta_t}^{\intercal}\mathcal{L}^{\text{noisy-meta}}(\hat{\mathbf{w}}^{t}(\Theta^{t}))\big|_{\zeta_t,\eta_t}\\
	    &\quad+ (1-\eta)^2\mathcal{K}(\hat{\bw}(\Theta^t))^{\intercal}\mathcal{K}(\hat{\bw}(\Theta^t)) \\
	    &\quad-2(1-\eta) \mathcal{K}(\hat{\bw}(\Theta^t))^{\intercal}\mathcal{L}^{\text{noisy-meta}}(\hat{\mathbf{w}}^{t}(\Theta^{t}))\big|_{\zeta_t,\eta_t}\Big]\\
	    &=\BE_{\zeta_t}\Big[(1-\eta) ||\frac{\partial\ell^{\cdot,\meta}(\hat{\bw})}{\partial \hat{\bw}}||^2 +(1-\eta)^2||\mathcal{K}(\hat{\bw}(\Theta^t))||^2\\
	    &\quad-2(1-\eta)^2 \mathcal{K}(\hat{\bw}(\Theta^t))^{\intercal} \frac{\partial\ell^{\cdot,\meta}(\hat{\bw})}{\partial \hat{\bw}}+\\
	    &\quad+\frac{2\eta}{K}\sum_c ||\frac{\partial\ell^{\meta}(c, f(\bx,\hat{\bw}))}{\partial \hat{\bw}}||^2  \\
	    &\quad-2\frac{\eta(1-\eta)}{K}\mathcal{K}(\hat{\bw}(\Theta^t))^{\intercal} \sum_c \frac{\partial\ell^{\meta}(c, f(\bx,\hat{\bw}))}{\partial \hat{\bw}} \Big]\\
	    &\leq \BE_{\zeta_t}\Big[||\nabla \mathcal{L}^{\meta}(\hat{\mathbf{w}}^{t}(\Theta^{t}))\big|_{\zeta_t}- \mathcal{K}(\hat{\bw}(\Theta^t))||^2\Big]\\
	    &\quad+ \frac{2\eta}{K}\sum_c ||\frac{\partial\ell^{\meta}(c, f(\bx,\hat{\bw}))}{\partial \hat{\bw}}||^2 \\
	     &= m\sigma^2 + 2\eta \rho^2
	\end{align*}
For a minibatch of size of $m$, the variance decreases as,
\begin{align}
       \BE_{\zeta_t, \eta_t}[||\xi^t||^2] \leq \frac{m}{m^2}(m\sigma^2 + \eta \rho^2)= \sigma^2 + \frac{2\eta \rho^2}{m}
\end{align}
	\end{proof}
With Lemma~\ref{lem:lemma1} and ~\ref{lem:lemma2}, we can now prove Theorem~\ref{thm:converge}. We will prove only for corrupted meta datasets. Proving convergence of clean meta dataset is easy and can be obtained by simply putting $\eta=0$ in the noisy result.
	\begin{theorem} \label{th1}
		Suppose the meta loss function $\ell$ and the classifier network loss $\ell_{\ce}^{\cdot, \train}$ is Lipschitz smooth with constant $L$, and have $\rho$-bounded gradients. The weighting function $\mathcal{W}(\cdot)$ has bounded gradient and twice differential with bounded Hessian. Let the classifier network learning rate $\alpha_t$ satisfies $\alpha_t=\min\{1,\frac{k}{T}\}$, for some $k>0, k<T$. The learning rate of the weighting network satisfies $\beta_t =\min\{\frac{1}{L},\frac{b}{\hat{\sigma}\sqrt{T}}\} $ for some $b>0$, such that $\frac{\hat{\sigma}\sqrt{T}}{b}\geq L$ where $\hat{\sigma}^2$ is the variance of drawing a minibatch corrupted with noise. Then Robust-Meta-Noisy-Weight-Net can achieve $\mathbb{E}[ \|\nabla \mathcal{L}^{\text{noisy-meta}}(\hat{\mathbf{w}}^{t}(\Theta^{t}))\|_2^2] \leq \epsilon$ in $\mathcal{O}(\frac{\hat{\sigma}^2}{(1-\eta)^2\epsilon^2})$ steps when meta loss function $\ell$ satisfies symmetric condition in Eq.~\ref{eq:sym}. In particular,
		\begin{align}
		\min_{0\leq t \leq T} \mathbb{E}[ \|\nabla \mathcal{L}^{\text{noisy-meta}}(\hat{\mathbf{w}}^{t}(\Theta^{t}))\|_2^2] \leq \mathcal{O}(\frac{\hat{\sigma}}{(1-\eta)\sqrt{T}}).
		\end{align}
	\end{theorem}
	\begin{proof}
		The update of $\Theta$ in each iteration is:
\begin{small}
		\begin{align*}
		\Theta^{t+1} =  \Theta^{t} -\beta_t \nabla \mathcal{L}^{\neta}(\hat{\mathbf{w}}^{t}(\Theta^{t}))\big|_{\zeta_t, \eta_t}.
		\end{align*}
		\end{small}
 We can rewrite the update equation as:
 \begin{small}
		\begin{align*}
		\Theta^{t+1} =  \Theta^{t} -\beta_t[ (1-\eta) \nabla \mathcal{L}^{\meta}(\hat{\mathbf{w}}^{t}(\Theta^{t}))+\xi^{t}],
		\end{align*}
		\end{small}
		where $\xi^{t} = \nabla \mathcal{L}^{\neta}(\hat{\mathbf{w}}^{t}(\Theta^{t}))\big|_{\zeta_t, \eta_t}-(1-\eta) \nabla \mathcal{L}^{\meta}(\hat{\mathbf{w}}^{t}(\Theta^{t}))$. %
		We have, $\mathbb{E}[\xi^{t}]=0$, as we have shown in Theorem~\ref{thm:main}.
		We have,
		\begin{small}
			\begin{align}\label{eq231}
			\begin{split}
			&\mathcal{L}^{\neta}(\hat{\mathbf{w}}^{t+1}(\Theta^{t+1}))-\mathcal{L}^{\neta}(\hat{\mathbf{w}}^{t}(\Theta^{t})) \\
			&= \left\{\mathcal{L}^{\neta}(\hat{\mathbf{w}}^{t+1}(\Theta^{t+1}))- \mathcal{L}^{\neta}(\hat{\mathbf{w}}^{t}(\Theta^{t+1}))\right\}\\
			&+\left\{\mathcal{L}^{\neta}(\hat{\mathbf{w}}^{t}(\Theta^{t+1}))-\mathcal{L}^{\neta}(\hat{\mathbf{w}}^{t}(\Theta^{t}))\right\}.
			\end{split}
			\end{align}
			\end{small}
Since meta loss function is Lipschitz smooth,  we have
\begin{small}
		\begin{align*}
			&\mathcal{L}^{\neta}(\hat{\mathbf{w}}^{t+1}(\Theta^{t+1}))- \mathcal{L}^{\neta}(\hat{\mathbf{w}}^{t}(\Theta^{t+1})) \\
			&\leq \langle \nabla \mathcal{L}^{\neta}(\hat{\mathbf{w}}^{t}(\Theta^{t+1})), \hat{\mathbf{w}}^{t+1}(\Theta^{t+1})-\hat{\mathbf{w}}^{t}(\Theta^{t+1}) \rangle\\
			&\quad \quad+ \frac{L}{2}\|\hat{\mathbf{w}}^{t+1}(\Theta^{t+1})-\hat{\mathbf{w}}^{t}(\Theta^{t+1})\|_2^2 
			\end{align*}
			\end{small}

Further, using the SGD update equation, we can write, $\hat{\mathbf{w}}^{t+1}(\Theta^{t+1})-\hat{\mathbf{w}}^{t}(\Theta^{t+1}) = - \alpha_t \frac{1}{n}
			\sum_{i=1}^n   \mathcal{W}(\ell_{\ce}^{i,\train}(\mathbf{w}^{t+1});\Theta^{t+1}) \nabla_{\mathbf{w}} \ell_{\ce}^{i,train}(\mathbf{w})\Big|_{\mathbf{w}^{t+1}}$. Thus, using the fact $\left\|\frac{\partial \ell_{\ce}^{i,\train} (\mathbf{w})}{\partial \mathbf{w}}\Big|_{\mathbf{w}^{t}}\right\|\leq \rho$, $\left\|\!\frac{\partial \ell^{j,\neta} (\hat{\mathbf{w}})}{\partial \hat{\mathbf{w}}}\!\Big|_{\hat{\mathbf{w}}^{t}}^T\!\right\|\!\!\leq\!\! \rho$, we can bound,
			\begin{small}
		\begin{align*}
		\|\mathcal{L}^{\neta}(\hat{\mathbf{w}}^{t+1}(\Theta^{t+1}))- \mathcal{L}^{\neta}(\hat{\mathbf{w}}^{t}(\Theta^{t+1})) \|
		\\\leq \alpha_t \rho^2+ \frac{L\alpha_t^2}{2} \rho^2 = \alpha_t\rho^2 (1+\frac{\alpha_t L}{2})
		\end{align*}
		\end{small}

		Since meta loss function is Lipschitz continuous from Lemma~\ref{lem:lemma1},  we get the following,
		\begin{small}
			\begin{align*}
			&\mathcal{L}^{\neta}(\hat{\mathbf{w}}^{t}(\Theta^{t+1}))-\mathcal{L}^{\neta}(\hat{\mathbf{w}}^{t}(\Theta^{t}))  \\
			&\leq \langle\nabla \mathcal{L}^{\neta}(\hat{\mathbf{w}}^{t}(\Theta^{t})),\Theta^{t+1}-\Theta^{t} \rangle + \frac{L}{2} \|\Theta^{t+1}-\Theta^{t}\|_2^2\\
			& = \langle\nabla \mathcal{L}^{meta}(\hat{\mathbf{w}}^{t}(\Theta^{t})), -\beta_t [(1-\eta)\nabla \mathcal{L}^{\neta}(\hat{\mathbf{w}}^{t}(\Theta^{t}))\\
			&\quad\quad+\xi^{t} ] \rangle + \frac{L\beta_t^2}{2} \|(1-\eta)\nabla \mathcal{L}^{\neta}(\hat{\mathbf{w}}^{t}(\Theta^{t}))+\xi^{t}\|_2^2\\
			& = -((1-\eta)\beta_t-(1-\eta)^2\frac{L\beta_t^2}{2}) \|\nabla \mathcal{L}^{\neta}(\hat{\mathbf{w}}^{t}(\Theta^{t}))\|_2^2 +\\
			&\quad\quad \frac{L\beta_t^2}{2}\|\xi^{t}\|_2^2 - (\beta_t-L(1-\eta)\beta_t^2)\langle \nabla \mathcal{L}^{\neta}(\hat{\mathbf{w}}^{t}(\Theta^{t})),\xi^{t}\rangle.
			\end{align*}
			\end{small}
		Using the above inequality, we bound Eq.(\ref{eq231}) as,
		\begin{small}
			\begin{align}
			\begin{split}
			&\mathcal{L}^{\neta}(\hat{\mathbf{w}}^{t+1}(\Theta^{t+1}))-\mathcal{L}^{\neta}(\hat{\mathbf{w}}^{t}(\Theta^{t})) \\
			&\leq \alpha_t\rho^2 (1+\frac{\alpha_t L}{2}) -((1-\eta)\beta_t-(1-\eta)^2\frac{L\beta_t^2}{2})\\
			&\quad\quad \times\|\nabla \mathcal{L}^{\neta}(\hat{\mathbf{w}}^{t}(\Theta^{t}))\|_2^2 + \frac{L\beta_t^2}{2}\|\xi^{t}\|_2^2 \\
			&\quad\quad- (\beta_t-L(1-\eta)\beta_t^2)\langle \nabla \mathcal{L}^{\neta}(\hat{\mathbf{w}}^{t}(\Theta^{t})),\xi^{t}\rangle.
			\end{split}
			\end{align}
			\end{small}
		We can simplifiy as,
		\begin{small}
			\begin{align*}
			\begin{split}
			&((1-\eta)\beta_t-(1-\eta)^2\frac{L\beta_t^2}{2})
			\|\nabla \mathcal{L}^{\neta}(\hat{\mathbf{w}}^{t}(\Theta^{t}))\|_2^2 \\
			&\leq  \alpha_t\rho^2 (1+\frac{\alpha_t L}{2})+\mathcal{L}^{\neta}(\hat{\mathbf{w}}^{t}(\Theta^{t}))\\
			&\quad\quad- \mathcal{L}^{\neta}(\hat{\mathbf{w}}^{t+1}(\Theta^{t+1}))
			+ \frac{L\beta_t^2}{2}\|\xi^{t}\|_2^2 \\
			&\quad\quad- (\beta_t-L(1-\eta)\beta_t^2)\langle \nabla \mathcal{L}^{\neta}(\hat{\mathbf{w}}^{t}(\Theta^{t})),\xi^{t}\rangle.
			\end{split}
		\end{align*}
		\end{small}
We can simplify as,
\begin{small}
			\begin{align}\label{eqrand}
			\begin{split}
			&\sum\nolimits_{t=1}^T ((1-\eta)\beta_t-(1-\eta)^2\frac{L\beta_t^2}{2}) \|\nabla \mathcal{L}^{\neta}(\hat{\mathbf{w}}^{t}(\Theta^{t}))\|_2^2 \\
			&\leq \mathcal{L}^{\neta}(\hat{\mathbf{w}}^{1})(\Theta^{1})-\mathcal{L}^{\neta}(\hat{\mathbf{w}}^{t+1}(\Theta^{t+1}))\\
			&\quad+\sum_{t=1}^T\alpha_t\rho^2 (1+\frac{\alpha_t L}{2})
-\!\sum_{t=1}^T\!(\beta_t\!-\!L(1-\eta)\beta_t^2)\\
&\quad\!\langle\nabla\! \mathcal{L}^{\neta}\!(\hat{\mathbf{w}}^{t}(\Theta^{t})),\!\xi^{t}\!\rangle\!+\!\frac{L}{2}\!\sum_{t=1}^T\!\beta_t^2\|\xi\!^{t}\!\|_2^2 \\
			&\leq \mathcal{L}^{\neta}\!(\hat{\mathbf{w}}^{1}(\!\Theta^{1}\!))\!+\sum_{t=1}^T\alpha_t\rho^2 (1+\frac{\alpha_t L}{2})\\
			&\quad-\!\sum_{t=1}^T \!(\beta_t\!-\!L(1-\eta)\beta_t^2)\!\langle\!\nabla\! \mathcal{L}^{\neta}(\hat{\mathbf{w}}^{t}(\Theta^{t})),\!\xi^{t}\!\rangle\!\\
			&\quad+\!\frac{L}{2}\!\sum_{t=1}^T\!\beta_t^2\|\xi^{t}\|_2^2,
			\end{split}
			\end{align}
			\end{small}
		Taking expectations with respect to $\xi^{N}$ on both sides of Eq. \ref{eqrand}, we get,
		\begin{small}
		\begin{align*}
			\sum_{t=1}^T &((1-\eta)\beta_t-(1-\eta)^2\frac{L\beta_t^2}{2})\mathbb{E}_{\xi^{N}} \|\nabla \mathcal{L}^{\neta}(\hat{\mathbf{w}}^{t}(\Theta^{t}))\|_2^2 \\
			&\leq \mathcal{L}^{\neta}\!(\hat{\mathbf{w}}^{1}(\!\Theta^{1}\!))\!+\sum_{t=1}^T\alpha_t\rho^2 (1+\frac{\alpha_t L}{2})+ \frac{L\hat{\sigma}^2}{2} \sum_{t=1}^T \beta_t^2,
		\end{align*}
		\end{small}
		since $\footnotesize{\mathbb{E}_{\xi^{N}} \langle \nabla \mathcal{L}^{\neta}(\Theta^{t}),\xi^{t}\rangle =0 }$ and $\mathbb{E} [\|\xi^{t}\|_2^2] \leq \hat{\sigma}^2$ (using Lemma~\ref{lem:lemma2},) where
		$\hat{\sigma}^2$ is the variance with respect to $\xi^{t}$.
		Finally, we can obtain the the bound as,
		\begin{small}
		\begin{align}
		\begin{split}
		&\min_{t} \mathbb{E} [ \|\nabla \mathcal{L}^{\neta}(\hat{\mathbf{w}}^{t}(\Theta^{t}))\|_2^2] \\
		&\leq \frac{\sum_{t=1}^T ((1-\eta)\beta_t-(1-\eta)^2\frac{L\beta_t^2}{2})\mathbb{E}_{\xi^{N}} \|\nabla \mathcal{L}^{\eta}(\hat{\mathbf{w}}^{t}(\Theta^{t}))\|_2^2}{\sum_{t=1}^T ((1-\eta)\beta_t-(1-\eta)^2\frac{L\beta_t^2}{2})}\\
		&\leq  \frac{1}{\sum_{t=1}^T (2(1-\eta)\beta_t-(1-\eta)^2L\beta_t^2)} \Big[2\mathcal{L}^{\neta}\!(\hat{\mathbf{w}}^{1}(\!\Theta^{1}\!))\\
		&\quad\quad+\sum_{t=1}^T\alpha_t\rho^2 (2+\alpha_t L\!) + L\hat{\sigma}^2\sum_{t=1}^T \beta_t^2\Big]\\
		& \leq \frac{1}{\sum_{t=1}^T 2(1-\eta)\beta_t} \Big[2\mathcal{L}^{\neta}\!(\hat{\mathbf{w}}^{1}(\!\Theta^{1}\!))\!+\sum_{t=1}^T\alpha_t\rho^2 (2+\alpha_t L) \\
		&\quad\quad+ L\hat{\sigma}^2\sum_{t=1}^T \beta_t^2\Big] \\
		& \leq \frac{1}{2T(1-\eta)\beta_t} \Big[2\mathcal{L}^{\neta}\!(\hat{\mathbf{w}}^{1}(\!\Theta^{1}\!))\!+ \alpha_1 \rho^2 T (2+ L) \\
		&\quad\quad+ L\hat{\sigma}^2 \sum_{t=1}^T \beta_t^2\Big]\\
		& = \frac{\mathcal{L}^{\neta}\!(\hat{\mathbf{w}}^{1}(\!\Theta^{1}\!))\!}{(1-\eta)T} \frac{1}{\beta_t} + \frac{2\alpha_1 \rho^2 (2+L)}{2(1-\eta) \beta_t}+\frac{L\hat{\sigma}^2}{2(1-\eta)T}\sum_{t=1}^T \beta_t\\
		& \leq  \frac{\mathcal{L}^{\neta}\!(\hat{\mathbf{w}}^{1}(\!\Theta^{1}\!))\!}{(1-\eta)T} \frac{1}{\beta_t} + \frac{2\alpha_1 \rho^2 (2+L)}{2(1-\eta)\beta_t}+L\hat{\sigma}^2 \beta_t\frac{1}{2(1-\eta)}\\
		& = \frac{\mathcal{L}^{\neta}\!(\hat{\mathbf{w}}^{1}(\!\Theta^{1}\!))\!}{(1-\eta)T} \max\{L,\frac{\hat{\sigma}\sqrt{T}}{b}\} \\
		&\quad\quad+ \min\{1,\frac{k}{T}\} \max\{L,\frac{\hat{\sigma}\sqrt{T}}{b}\} \frac{\rho^2  (2+L)}{2(1-\eta)}\\
		&\quad\quad+ L \hat{\sigma}^2 \min\{\frac{1}{L},\frac{b}{\hat{\sigma}\sqrt{T}}\}\frac{1}{2(1-\eta)}\\
		& \leq \frac{\hat{\sigma}\mathcal{L}^{\neta}\!(\hat{\mathbf{w}}^{1}(\!\Theta^{1}\!)\!}{(1-\eta)b\sqrt{T}}+ \frac{k\hat{\sigma} \rho^2(2+L)}{b(1-\eta) \sqrt{T}} + \frac{L\hat{\sigma} b}{(1-\eta)\sqrt{T}} \\
		&= \mathcal{O}(\frac{\hat{\sigma}}{(1-\eta)\sqrt{T}}).
		\end{split}
		\end{align}
		\end{small}
		The third inequality holds for $\sum_{t=1}^T (2\beta_t-L\beta_t^2) \geq \sum_{t=1}^T \beta_t$.
		Therefore, Robust-Meta-Noisy-Weight-Network can achieve $\min_{0\leq t \leq T} \mathbb{E}[ \|\nabla \mathcal{L}^{\neta}(\Theta^{t})\|_2^2] \leq \mathcal{O}(\frac{\hat{\sigma}}{(1-\eta)\sqrt{T}})$ in $T$ steps.
	\end{proof}

\section{Additional Results}
	\begin{figure}[pt] \vspace{2mm}
			\includegraphics[width=0.48\textwidth]{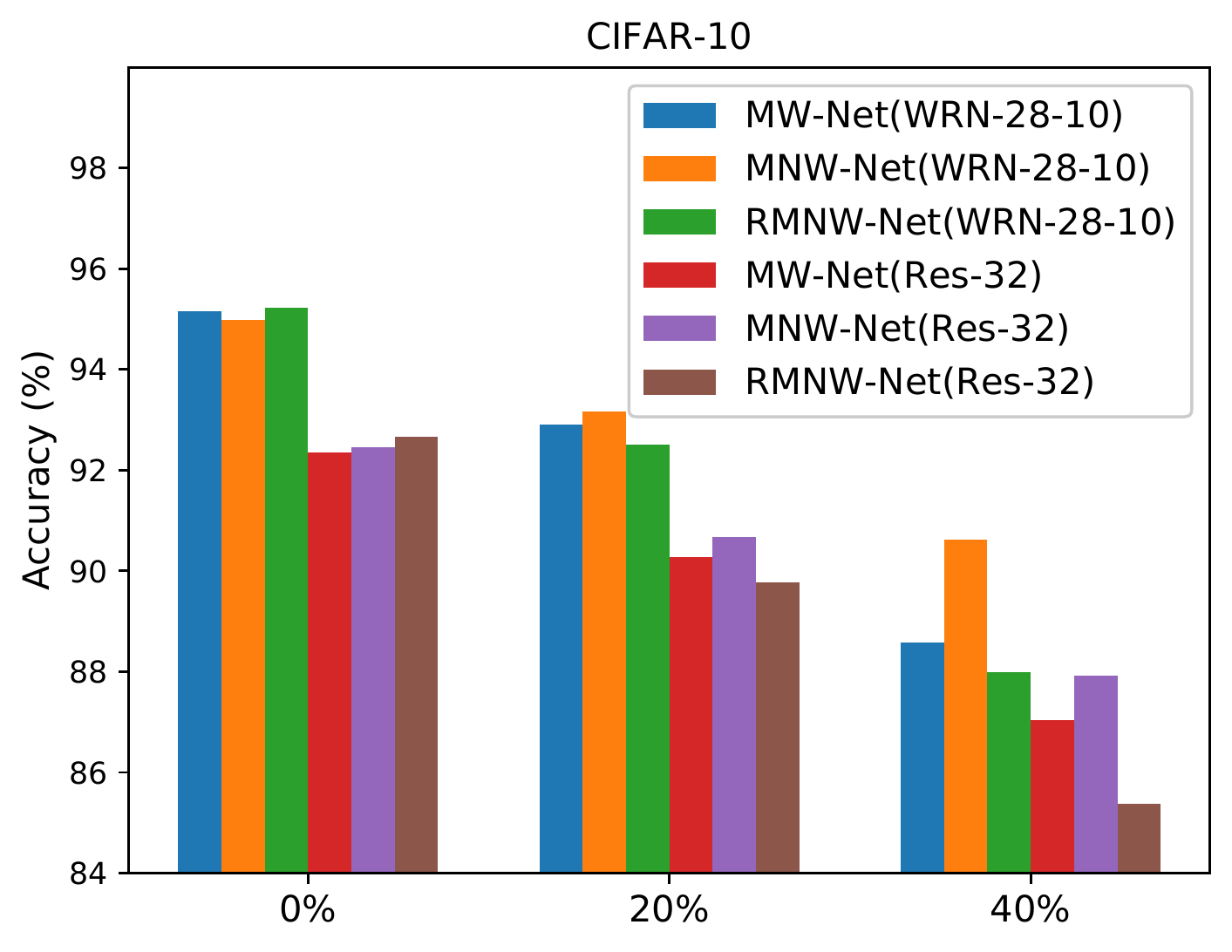} \\
			\includegraphics[width=0.48\textwidth]{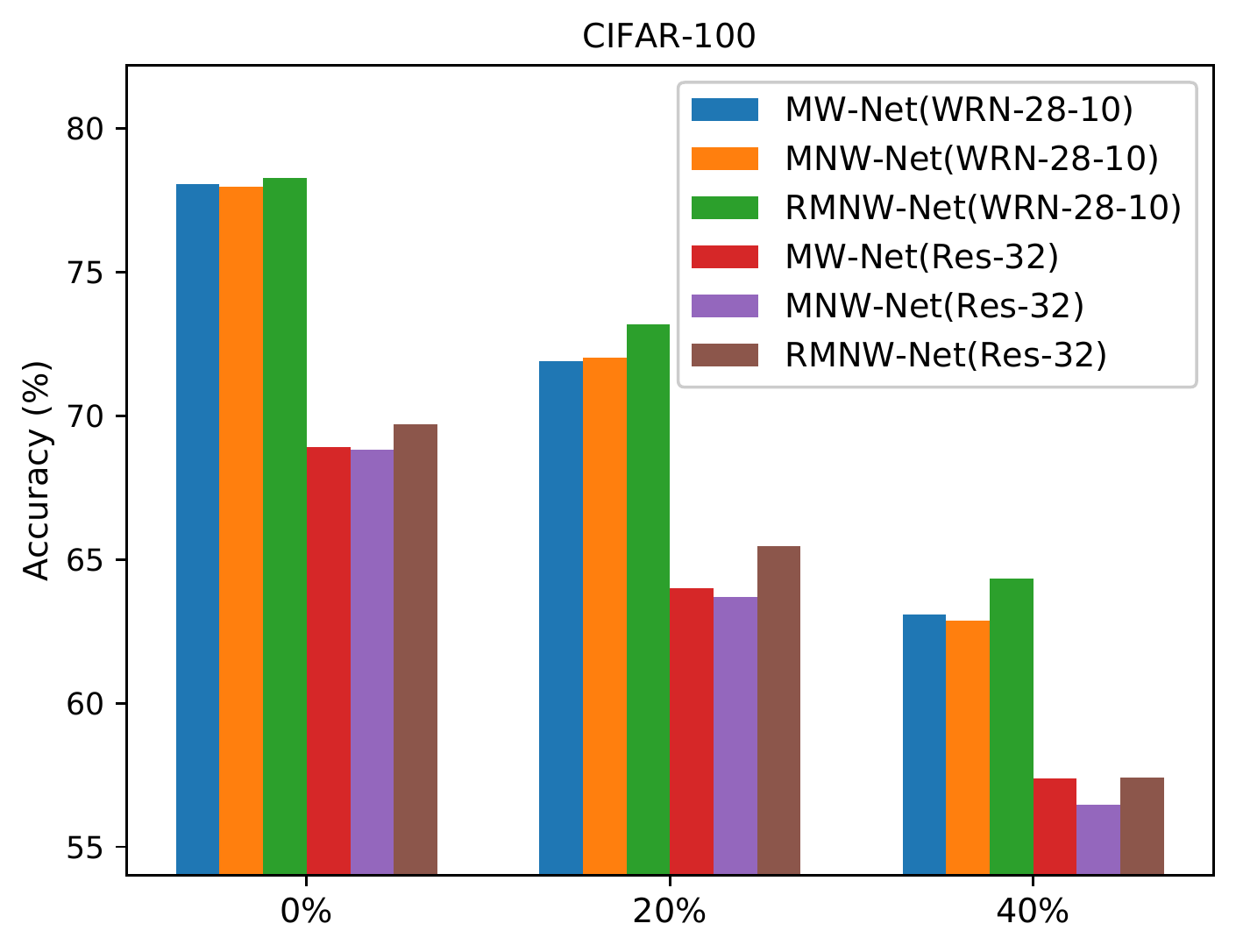}\vspace{0mm}
			\caption{Performance comparison for different classifier architecture (WRN-28-10 and ResNet32) for CIFAR flip2 noise \cite{mwnet}}
			\label{fig:wide-vs-res}
	\end{figure}
\textbf{Indifference to network architecture on flip2 noise model}
In Table ~\ref{tab:main}, we experimented with ResNet-32 architecture for flip2 noise model following \cite{mwnet}. For parity with uniform noise, we also experiment with WRN-28-10 architecture for flip2 noise.  
Figure~\ref{fig:wide-vs-res} shows performances of MW-Net$^{\ast}$, MNW-Net, and RMNW-Net on these two architectures for flip2 noise model. We observe similar trends on both architectures for the flip2 noise model. On CIFAR-10/100, MNW-Net/RMNW-Net performs marginally better than the other on both architectures.